\documentclass[10pt]{article} 
\usepackage[preprint]{tmlr}

\usepackage{amsmath,amsfonts,bm}

\def\eqref#1{equation~\ref{#1}}

\def\1{\bm{1}}

\DeclareMathAlphabet{\mathsfit}{\encodingdefault}{\sfdefault}{m}{sl}
\SetMathAlphabet{\mathsfit}{bold}{\encodingdefault}{\sfdefault}{bx}{n}

\usepackage{url}
\usepackage{hyperref}
\hypersetup{hidelinks}
\usepackage{xcolor}
\newcommand{\best}[1]{\textcolor{red}{\textbf{#1}}}
\newcommand{\second}[1]{\textcolor{blue}{\underline{#1}}}
\newcommand{\model}{GRACE}
\usepackage{color,soul}
\usepackage{enumitem}
\usepackage{multirow}   
\usepackage{booktabs}   
\usepackage{algorithm}
\usepackage{algpseudocode}
\usepackage{graphicx}   

\usepackage[most]{tcolorbox}
\usepackage{amsthm}

\newtheorem{theorem}{Theorem}
\newtheorem{lemma}{Lemma}

\newtcolorbox{questionbox}{
    colback=gray!5,
    colframe=gray!50,
    boxrule=0.8pt,
    arc=2mm,
    left=4mm,
    right=4mm,
    top=2mm,
    bottom=2mm
}

\title{Learning Granger Causality under Latent Confounding via\\ Intervention-Induced Heterogeneity}

\author{\name Ziyi Zhang \email zyzhang@tamu.edu \\
      \addr Department of Electrical and Computer Engineering,
      Texas A\&M University
      \AND
      \name Shaogang Ren \email shaogang-ren@utc.edu \\
      \addr Department of Computer Science and Engineering,
      University of Tennessee at Chattanooga
      \AND
      \name Xiaoning Qian \email xqian@tamu.edu\\
      \addr Department of Electrical and Computer Engineering,
      Texas A\&M University\\
      Brookhaven National Laboratory
      \AND
      \name Nick Duffield \email duffieldng@tamu.edu\\
      \addr Department of Electrical and Computer Engineering,
      Texas A\&M University
      }

\def\month{MM}  
\def\year{YYYY} 
\def\openreview{\url{https://openreview.net/forum?id=XXXX}} 

\begin{document}

\maketitle
\begin{abstract}
Granger causality characterizes directed predictive dependencies in multivariate time series, but recovering such dependencies becomes challenging in the presence of latent confounding. Cross-environment invariance provides a natural source of information in heterogeneous settings, yet invariance alone can be insufficient: when latent-to-observed mechanisms remain stable, hidden confounders can induce predictive dependencies that are just as invariant as genuine Granger-causal relations. We show that interventions provide an additional source of identifying information by inducing structured variation in observed mechanisms, while stable latent pathways need not exhibit the same cross-environment changes. In practice, however, neither the intervened environments nor the affected mechanisms are known. We propose GRACE, a framework for learning Granger causality under latent confounding from intervention-induced heterogeneity. GRACE decomposes multivariate dynamics into a shared Granger mechanism, sparse environment-specific deviations that capture edge-level interventions, and a latent component that accounts for confounding. Under a linear generative model, we show that GRACE can recover which environments intervene on a given edge when the edge is perturbed in at least one but fewer than half of the environments and the intervention effect is sufficiently large to survive sparsity shrinkage; the recovered intervention pattern then provides a certificate for the corresponding Granger causal edge. Experiments on synthetic and real-world time series demonstrate improved Granger causal structure recovery under latent confounding and unknown interventions.
\end{abstract}

\section{Introduction}
Granger causality has been widely used to characterize directed predictive dependencies in time series, with applications across finance, healthcare, and retail systems. With the development of neural networks~\citep{lecun2015deep}, there has been growing interest in extending classical vector autoregressive approaches toward more expressive nonlinear models for Granger causal discovery~\citep{khanna2019economy,marcinkevivcs2021interpretable,tank2021neural,cheng2023cuts,zhou2024jacobian,cheng2024cuts+,han2025root}. Yet reliable recovery of Granger causal structure remains challenging in realistic settings, where observed time series are subject to latent confounding and distributional heterogeneity across environments.

\begin{figure*}[t]
    \centering
    \includegraphics[width=\linewidth]{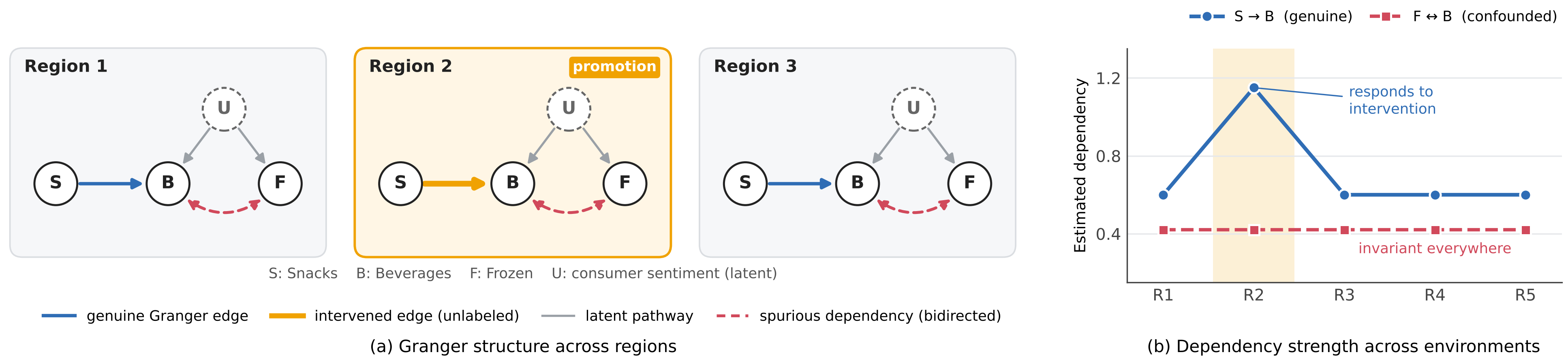}
    \caption{Motivating example. \textbf{(a)} Latent sentiment $U$ induces a spurious dependency $F \leftrightarrow B$; an unlabeled promotion in Region~2 perturbs only the genuine edge $S \to B$. \textbf{(b)} Both dependencies are invariant without intervention; only the genuine edge responds to it.}
    \label{fig:motivation}
\end{figure*}

Environmental heterogeneity is often viewed as an asset for structure learning. A common principle is cross-environment invariance: dependencies that remain stable across environments are considered more indicative of shared underlying mechanisms than associations that vary with environmental conditions~\citep{peters2016causal,heinze2018invariant,arjovsky2019invariant,pfister2019invariant,scholkopf2021toward}. Latent confounding, however, can undermine this principle. Consider retail demand across regions (see Figure~\ref{fig:motivation}): an unobserved factor such as regional consumer sentiment may jointly influence the sales of two product categories, inducing a predictive dependency between them even in the absence of a direct Granger-causal link. Crucially, the cross-environment behavior of this dependency is determined by the latent-to-observed pathway rather than by the observed Granger causal mechanism itself. If the latent pathway remains stable, the resulting confounded dependency can be just as invariant as a direct Granger-causal relation. In this regime, the ambiguity is not merely an estimation problem that can be resolved with more data or more environments: when both direct and confounded dependencies persist across environments, invariance alone provides no criterion for distinguishing them. What is needed instead is cross-environment variation that affects observed mechanisms and latent pathways differently.

Interventions provide precisely this form of variation. When an intervention perturbs an observed mechanism, for example, a promotion that changes how the sales of one product category influence another in a subset of regions, the dependency associated with the affected causal edge changes in those environments, whereas a stable latent pathway need not exhibit the same shift. The identifying signal therefore lies not in invariance alone, but in the structured pattern of deviations around an otherwise shared mechanism. In practice, however, such interventions are rarely labeled: one typically observes only a collection of heterogeneous time series, without knowing which environments were perturbed or which edge-level mechanisms were affected.

Existing approaches address only parts of this setting. Some assume observational or stationary data and do not model interventions~\citep{huang2020causal}; others require labels identifying non-intervened environments~\citep{han2025root}, or known intervention timings or targets~\citep{gao2022idyno,liu2023causal}. Methods for heterogeneous environments or unknown interventions~\citep{squires2023causal,mameche2024identifying,zhang2024learning} provide complementary perspectives, but latent confounding, unknown intervention patterns, and Granger causal structure learning are typically handled separately. Existing formulations therefore do not directly address the setting in which unlabeled intervention-induced variation itself serves as identifying information against invariant confounding. This gap suggests that the central challenge is not merely to model heterogeneous environments, but to determine whether the unknown intervention pattern itself can provide identifying information in the presence of stable latent confounding.

\begin{questionbox}
\centering
\itshape
Can intervention-induced heterogeneity be leveraged to recover Granger causal structure under latent confounding when intervention environments and targets are unknown?
\end{questionbox}

We answer this question with GRACE, a framework for learning 
\underline{Gra}nger causality from heterogeneous time series in the presence of latent 
\underline{C}onfounding and unknown int\underline{E}rventions. GRACE decomposes the dynamics of each environment 
$k \in \{1,\dots,N\}$ into three complementary components: a shared Granger mechanism across environments, sparse environment-specific deviations capturing edge-level interventions, and latent representations accounting for hidden confounding. Neither the intervention deviations nor their support is observed. These components are jointly learned through \emph{Latent Confounder Absorption Modules} (LCAMs), \emph{Intervention Detection Networks} (IDNs), and an \emph{Invariant Granger Causal Network} (IGCN), whose representations are integrated by a next-timestep embedding and prediction network. While GRACE supports nonlinear predictive mechanisms in practice, our identifiability analysis is developed under a linear generative model. Importantly, intervention detection is not merely an auxiliary task: the inferred intervention pattern provides the cross-environment contrast used to distinguish direct Granger-causal dependencies from invariant confounded associations.

This perspective yields a two-sided identification principle. Let $\mathcal{I}_{ij} \subseteq \{1,\dots,N\}$ denote the environments in which edge $j\to i$ is intervened. If $\mathcal{I}_{ij}=\emptyset$, there is no intervention-induced contrast against an invariant confounded dependency. If interventions affect at least half of the environments, the unperturbed regime no longer forms a strict majority. Under our linear analysis, when $1\leq|\mathcal{I}_{ij}|<N/2$ and the intervention deviations are sufficiently separated from the shrinkage threshold, the sparse shared-plus-deviation decomposition used by GRACE recovers the unknown intervention support. The recovered deviations then provide edge-level evidence through the structure-preserving intervention assumption. This condition captures the need for both contrast and reference. Our contributions are summarized as follows:

\begin{itemize}
    \item We reformulate Granger causal discovery under latent confounding as the joint recovery of a shared Granger causal structure and environment-specific intervention deviations, and show that cross-environment invariance alone cannot separate direct Granger relations from stable confounding.
    \item We establish a sufficient per-edge recovery condition, $1 \leq |\mathcal{I}_{ij}| < N/2$, together with an intervention-magnitude separation condition, under which GRACE recovers unknown intervention environments and uses the recovered deviations to certify the corresponding causal edge.
    \item We propose GRACE to jointly infer shared Granger causal structure and unknown edge-level intervention deviations while accounting for latent confounding, and validate it on synthetic and real-world time series data.
\end{itemize}

\section{Related Work}
\noindent\textbf{Methods for Granger Causal Discovery.}
Traditional Granger causal discovery methods are mostly built upon vector autoregressive (VAR) models~\citep{granger1969investigating}. \cite{arnold2007temporal} introduced Granger causality inference using LASSO~\citep{tibshirani1996regression}, while \cite{tank2021neural} extended this approach through a sparse-input multi-layer perceptron (MLP) and a long short-term memory (LSTM) network. \cite{khanna2019economy} leveraged the efficiency of the economical statistical recurrent unit (eSRU) architecture, incorporating Group-LASSO~\citep{yuan2006model} regularization at the input layer. \cite{marcinkevivcs2021interpretable} proposed a generalized vector autoregressive (GVAR) model that improves neural network interpretability with sign detection. \cite{cheng2023cuts,cheng2024cuts+} developed approaches to infer Granger causality from irregular time series data. \cite{zhou2024jacobian} introduced a Jacobian regularizer-based framework that constructs a single efficient model to predict all target variables simultaneously. \cite{han2025root} proposed an encoder-decoder architecture that effectively leverages anomalous time series data to infer Granger causality. However, most of these methods focus on observational or stationary settings and do not explicitly address the joint presence of latent confounding and heterogeneous interventions, limiting their ability to distinguish direct Granger-causal dependencies from confounded or environment-specific predictive associations.

\noindent\textbf{Methods for Latent Confounding and Interventions.}
When causal sufficiency does not hold, classical approaches such as the Fast Causal Inference (FCI) family~\citep{spirtes2000causation,zhang2008completeness,colombo2012learning,ogarrio2016hybrid} and optimization-based methods~\citep{chandrasekaran2010latent} have been developed to account for latent confounding; in the time series setting, tsFCI~\citep{entner2010causal}, SVAR-FCI~\citep{malinsky2018causal}, and LPCMCI~\citep{gerhardus2020high} extend this line of work to temporal data. A complementary line of research exploits heterogeneous environments, either through cross-environment invariance~\citep{peters2016causal,heinze2018invariant,pfister2019invariant} or by modeling interventions, including methods for perfect and imperfect interventions with known or unknown targets~\citep{brouillard2020differentiable,yang2018characterizing,squires2023causal,mameche2024identifying}, heterogeneous or non-stationary data~\citep{huang2020causal}, observational and interventional time series with known targets~\citep{gao2022idyno}, latent intervened domains~\citep{liu2023causal}, and unknown interventions in time series~\citep{zhang2024learning}. Despite these advances, existing approaches typically address latent confounding, environmental heterogeneity, and unknown interventions separately. Our work instead leverages intervention-induced cross-environment variation to distinguish shared Granger-causal relations from latent confounding and intervention-specific effects, and characterizes when such heterogeneity can certify the underlying Granger causal structure.

\section{Preliminaries}
\subsection{Granger Causality}
Granger causality was originally proposed in the bivariate setting~\citep{granger1969investigating} to test whether the history of one time series improves the prediction of another beyond what is explained by its own history. It has since been extended to the multivariate setting, where the conditional prediction accounts for the histories of all observed variables simultaneously~\citep{tank2021neural, khanna2019economy, zhou2024jacobian}. Consider a stationary multivariate time series $X=\{X_1,\ldots,X_T\}$ with $T$ timesteps, where each $X_t \in \mathbb{R}^{d}$. Assume that lagged dependencies between variables are represented by the following structural predictive model:
\begin{equation}
    X_{t+1}^{i} = f_i(X^{1}_{1:t},\ldots,X^{d}_{1:t}) + \epsilon_{t+1}^{i} \hspace{0.2cm} \text{for} \hspace{0.1cm} 1 \leq i \leq d,
\end{equation}
where $f_i$ specifies how the observed histories map to the next value of series $i$. Time series $j$ is Granger non-causal for time series $i$ if $f_i(X^{1}_{1:t},\ldots,X^{j}_{1:t},\ldots,X^{d}_{1:t}) = f_i(X^{1}_{1:t},\ldots,\tilde{X}^{j}_{1:t},\ldots,X^{d}_{1:t})$ for all $X^{1}_{1:t},\ldots,X^{d}_{1:t}$ and all $X_{1:t}^{j} \neq \tilde{X}_{1:t}^{j}$. Under the lagged structural model used below, a nonzero structural coefficient from $X_t^j$ to $X_{t+1}^i$ corresponds to a direct Granger-causal relation $j\to i$.

\subsection{Structural Causal Model with Latent Confounders in Time Series}
The standard structural causal models (SCMs) for a set of observed variables $X$, assuming no latent confounders, can be described as $X = W^{\top}X + \epsilon$, where $W$ is the weighted adjacency matrix encoding the causal relationships among variables in $X$, and each noise term $\epsilon_{i}$ is assumed to be independent of the parents $\text{PA}(X_i)$ of variable $X_i$. Assuming no instantaneous effects, the SCMs can be adapted into a first-order vector autoregressive form, $X_{t+1} = W^{\top} X_{t} + \epsilon_{t+1}$, where $W$ captures the causal relationships from $X_{t}$ to $X_{t+1}$. To incorporate latent confounders in time series, we augment the observed state with latent variables $Z_t \in \mathbb{R}^{p}$:
\begin{equation}\label{eq:data_gen}
    \begin{pmatrix} X_{t+1} \\ Z_{t+1} \end{pmatrix} = 
    \begin{pmatrix} 
    W_{X_{t+1}X_{t}}^{\top} & W_{X_{t+1}Z_{t}}^{\top} \\ 
    0 & W_{Z_{t+1}Z_{t}}^{\top} 
    \end{pmatrix} 
    \begin{pmatrix} X_{t} \\ Z_{t} \end{pmatrix} +
    \begin{pmatrix} \epsilon_{t+1}^X \\ \epsilon_{t+1}^Z 
    \end{pmatrix},
\end{equation}
where we assume $W_{Z_{t+1}X_{t}} = 0$, i.e., observed variables do not affect latent ones, and that each latent variable depends only on its own past, so that $W_{Z_{t+1}Z_{t}}$ is diagonal. This yields the component-wise equations,
\begin{align}
    X_{t+1} &= W_{X_{t+1}X_{t}}^{\top}X_{t} + 
    W_{X_{t+1}Z_{t}}^{\top}Z_{t} + 
    \epsilon_{t+1}^{X}, \label{eq:x_gen}\\
    Z_{t+1} &= W_{Z_{t+1}Z_{t}}^{\top}Z_{t} + 
    \epsilon_{t+1}^{Z}. \label{eq:z_gen}
\end{align}
Here, $W_{X_{t+1}X_{t}}$ encodes the Granger causality among observed variables, $W_{X_{t+1}Z_{t}}$ captures the influence of latent confounders on observed variables, and $W_{Z_{t+1}Z_{t}}$ governs the latent dynamics. To generalize this framework, we extend Eqs.~(\ref{eq:x_gen})--(\ref{eq:z_gen}) as follows,
\begin{align}
    X_{t+1} &= f_{X}(W_{X_{t+1}X_{t}}^{\top}X_{t} + 
    W_{X_{t+1}Z_{t}}^{\top}Z_{t}) + 
    \epsilon_{t+1}^{X},\\
    Z_{t+1} &= f_{Z}(W_{Z_{t+1}Z_{t}}^{\top}Z_{t}) + 
    \epsilon_{t+1}^{Z},
\end{align}
where $f_X$ and $f_Z$ can be either linear or non-linear functions.
\subsection{Types of Interventions in Time Series}
\begin{figure*}[t]
    \centering
    \includegraphics[width=\linewidth]{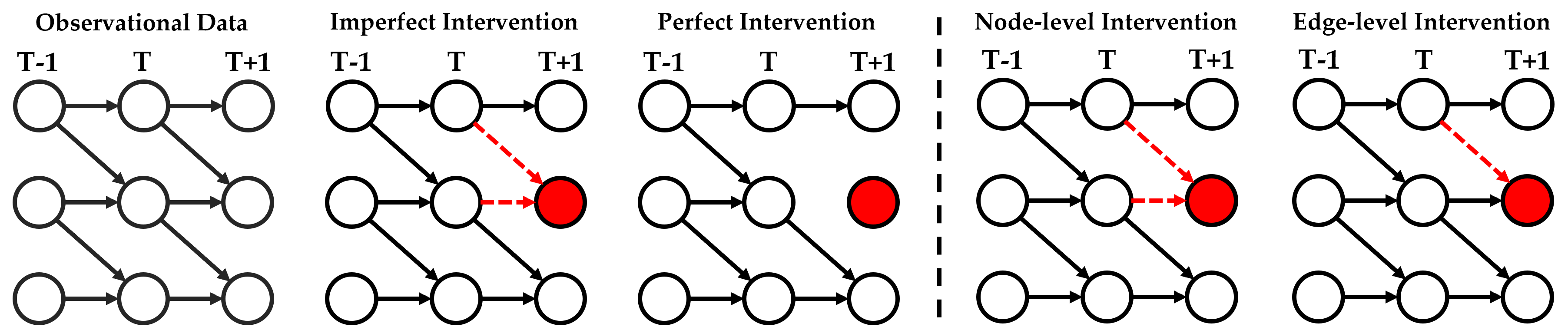}
    \caption{(Left): Imperfect interventions alter all causal relations from the parent nodes to the target node, whereas perfect interventions disconnect the target node from its parents; (Right): Node-level interventions apply to all edges, whereas edge-level interventions generalize to any subset.}
    \label{fig:types}
\end{figure*}
In time series, an intervention on a variable $X_{t}^{i}$ is defined by replacing its conditional distribution $p(X_{t}^{i}|\text{PA}(X_{t}^{i}))$ with a modified distribution $\tilde{p}(X_{t}^{i}|\text{PA}(X_{t}^{i}))$, where $\text{PA}(X_{t}^{i})$ denotes its set of parents. Interventions can be classified as imperfect (soft or parametric) interventions, with perfect (hard or structural) interventions being a special case in which all parental influence is entirely removed, i.e., $\tilde{p}(X_{t}^{i}|\text{PA}(X_{t}^{i})) =\tilde{p}(X_{t}^{i})$~\citep{brouillard2020differentiable, peters2016causal}. Interventions can also be categorized by granularity. Node-level interventions are applied to a target node and alter or remove all edges connecting it to its parents, representing a special case of edge-level interventions, which may affect either all or only a subset of parent--child relations~\citep{zhang2024learning}. Both types of categorization are illustrated in Figure~\ref{fig:types}.

\subsection{Problem Setting}\label{sec:setting}
We observe time series from $N$ environments, $\{X_k\}_{k=1}^{N}$, each generated by Eqs.~(\ref{eq:x_gen})-(\ref{eq:z_gen}) with environment-specific observed dynamics
\begin{equation}\label{eq:env}
    W_{X_{t+1}X_{t}}^{(k)} = W_{X_{t+1}X_{t}} + \Delta_k,
\end{equation}
where $W_{X_{t+1}X_{t}}$ is the shared Granger causal structure and $\Delta_k$ is a sparse deviation induced by unknown edge-level interventions. For the base problem setting, the latent pathway is shared across environments: $W_{X_{t+1}Z_{t}}$ and $W_{Z_{t+1}Z_{t}}$ are common across $k$. Because interventions can change the covariance structure of $X_t$, sharing these latent structural mechanisms does not by itself guarantee that the confounding contribution to the population regression is invariant across environments. Our theoretical analysis therefore directly assumes stability of this induced contribution, $\Gamma_k=\Gamma$, as formalized in Appendix~\ref{app:theory}. For each edge $j \to i$, let $\mathcal{I}_{ij} \subseteq \{1,\dots,N\}$ denote the set of environments $k$ in which the edge is intervened, i.e., in which the entry of $\Delta_k$ corresponding to $j \to i$ is nonzero. Neither $\Delta_k$ nor $\mathcal{I}_{ij}$ is observed. The goal is to recover the support of $W_{X_{t+1}X_{t}}$ and the intervention pattern $\{\mathcal{I}_{ij}\}$ from the observed series alone.

\section{Methodology}
\label{mtd}
\begin{figure*}[ht]
    \centering
    \includegraphics[width=\linewidth]{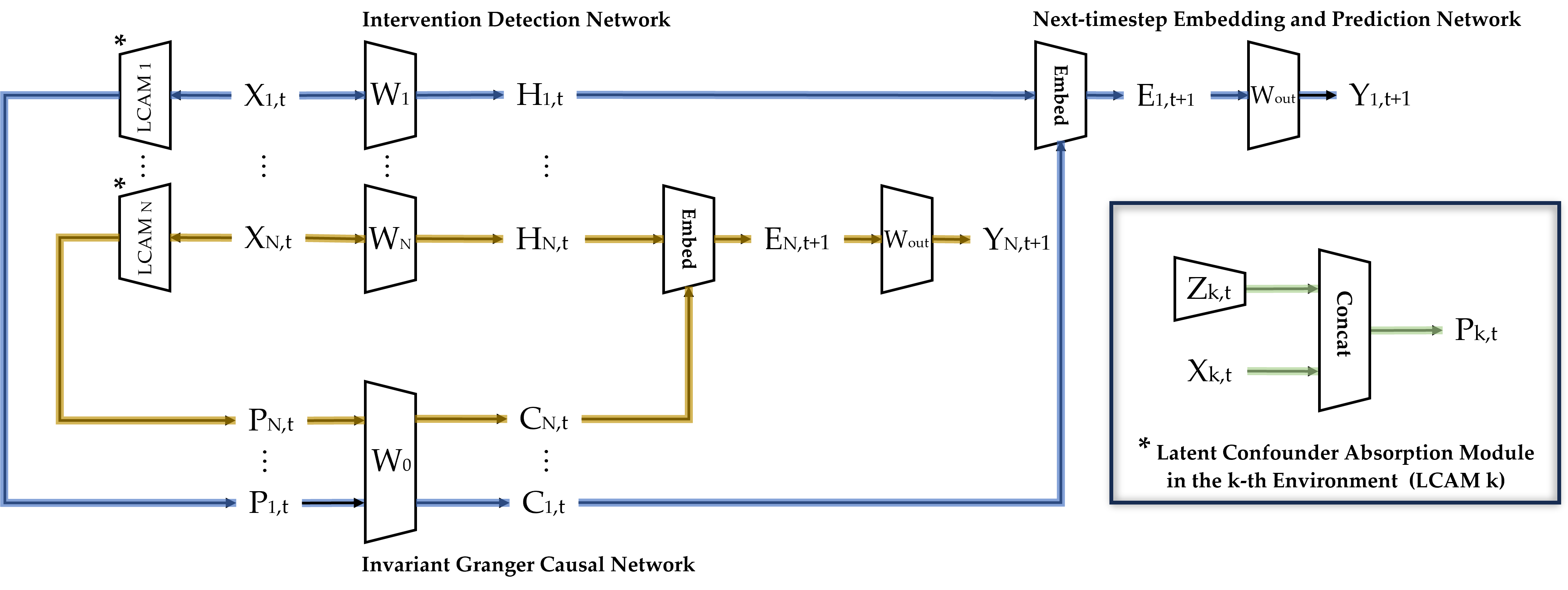}
    \caption{The framework of \model{}, composed of latent confounder absorption modules, intervention detection networks, an invariant Granger causal network, and next-timestep embedding and prediction networks.}
    \label{fig:arch}
\end{figure*}

In this section, we introduce the \model{} framework. The key design principle is to decompose the prediction of each variable into three components: a shared Granger mechanism, an environment-specific intervention component that captures deviations from the shared mechanism, and a latent component intended to account for variation induced by hidden factors. By jointly optimizing these components across environments, \model{} estimates a shared Granger causal structure together with edge-level intervention deviations. We present the model architecture, optimization, and theoretical analysis in the following sections.
\subsection{Model Architecture}
Given heterogeneous multivariate time series data $X \in \mathbb{R}^{N \times d \times T}$ collected from $N$ environments, we denote $X = \{X_1, X_2, \dots, 
X_N\}$, where $X_k \in \mathbb{R}^{d \times T}$ represents the multivariate time series from the $k$-th environment, consisting of $d$ observed variables over $T$ timesteps.

\noindent\textbf{Latent Confounder Absorption Modules (LCAMs).} 
For each environment $k$ and timestep $t$, with observed variables $X_{k,t} \in \mathbb{R}^{d}$, we introduce a learnable vector $Z_{k,t} \in \mathbb{R}^{p}$ to represent the latent confounders, where $p$ denotes the number of latent confounders. Intuitively, $Z_{k,t}$ provides a regularized latent representation intended to absorb residual variation in $X_{k,t+1}$ attributable to hidden factors, reducing the extent to which such variation is assigned to observed Granger causal edges. The resulting input vector is constructed by concatenating the observed variables with the latent confounders:
\begin{equation}
    P_{k,t} = \text{LCAM}_k(X_{k,t}, 
    Z_{k,t}) = [X_{k,t}, Z_{k,t}] \in \mathbb{R}^{d+p}.
\end{equation}
\noindent\textbf{Intervention Detection Networks (IDNs).} Since interventions are generally applied to observed variables, for each input $X_{k,t}$ the intervention detection network applies a linear projection to produce an environment-specific intervention representation:
\begin{equation}
    H_{k,t} = W_kX_{k,t} 
    \in \mathbb{R}^{m}.
\end{equation}
\noindent\textbf{Invariant Granger Causal Network (IGCN).} All concatenated inputs $P_t = \{P_{1,t}, P_{2,t}, \dots, P_{N,t}\}$ from the LCAMs share a common invariant Granger causal network, which projects each $P_{k,t}$ into an invariant causal representation:
\begin{equation}
    C_{k,t} = W_0P_{k,t} 
    \in \mathbb{R}^{m}.
\end{equation}
The shared weight $W_0$ encodes causal dependencies that are invariant across all environments: its observed-variable block corresponds to the shared Granger causal structure $W_{X_{t+1}X_{t}}$ to be recovered, and its latent block to $W_{X_{t+1}Z_{t}}$. The environment-specific weight $W_k$ corresponds to the intervention deviation $\Delta_k$ in Eq.~(\ref{eq:env}).

\noindent\textbf{Next-timestep Embedding and Prediction Network.} To generate the next-timestep prediction, we combine the intervention representation $H_{k,t}$ with the invariant causal representation $C_{k,t}$ via an embedding function $\text{Embed}(\cdot)$:
\begin{equation}
    E_{k,t+1} = \text{Embed}(H_{k,t}, 
    C_{k,t}) \in \mathbb{R}^{s}.
\end{equation}
The resulting embedding is linearly projected to generate the next-timestep prediction:
\begin{equation}
    \hat{X}_{k,t+1} = W_{\text{out}}
    E_{k,t+1} \in \mathbb{R}^{d}.
\end{equation}
In the linear case, $m{=}s{=}d$, $\text{Embed}(\cdot)$ is the sum and $W_{\text{out}}$ is the identity, so that $\hat{X}_{k,t+1} = W_0P_{k,t} + W_kX_{k,t}$.
\subsection{Optimization}
For each variable $i$, we denote by $W_0^{i} \in \mathbb{R}^{1 \times (d+p)}$ the invariant weight vector mapping $P_{k,t}$ to $X_{k,t+1}^i$, capturing time-lagged causal dependencies that remain invariant across environments. For each environment $k$, we define $W_k^{i} \in \mathbb{R}^{1 \times d}$ as the edge-level intervention weight vector from $X_{k,t}$ to $X_{k,t+1}^i$. We propose a unified loss function that jointly optimizes over $W$ and $Z$:
\begin{equation}\label{eq:ll}
\begin{aligned}
    \min_{W,Z} \sum_{k=1}^N \sum_{i=1}^d \sum_{t=1}^{T-1} 
    &\big\| X_{k,t+1}^{i} - (W_{0}^{i}P_{k,t} + 
    W_{k}^{i}X_{k,t}) \big\|^2_2
    + (1 - \alpha) \sum_{i=1}^d\sum_{j=1}^{d+p} 
    \big\| (W_{0,j}^i, W_{1,j}^i,\ldots,W_{N,j}^i) 
    \big\|_2\\
    &+ \alpha\sum_{k=1}^N \sum_{i=1}^d \sum_{j=1}^{d}
    \big\| W_{k,j}^i \big\|_2 + \lambda_z \sum_{k=1}^N 
    \sum_{l=1}^{p} \sqrt{ \frac{1}{T}\sum_{t=1}^T 
    Z^2_{k,l,t}},
\end{aligned}
\end{equation}
where $W_{k,j}^{i} := 0$ for $j > d$, so that the groups over the latent block contain only $W_{0,j}^{i}$. The first term is the prediction loss, the second term encourages group sparsity across environments to identify invariant causal edges, the third term promotes sparsity of environment-specific intervention weights, and the fourth term regularizes the latent confounder matrix with $\lambda_z > 0$. The hyperparameter $\alpha \in (0,1)$ balances sparsity across groups and within groups. To generalize the method to non-linear settings, we assume the existence of functions $f_i(\cdot)$, $g_{k,i}(\cdot)$, and $h_{k,i}(\cdot)$ such that:
\begin{equation}
\mathbb{E}[X_{k,t+1}^{i}|\text{PA}(X_{k,t+1}^{i}), 
Z_{k,t}] = h_{k,i}\big(f_{i}(P_{k,t}),g_{k,i}(X_{k,t})
\big),
\end{equation}
where $f_i(\cdot)$ is the invariant mechanism shared across all environments, $g_{k,i}(\cdot)$ encodes edge-level interventions acting on variable $i$ in environment $k$, and $h_{k,i}(\cdot)$ is an aggregation function integrating the two components. The overall objective is then reformulated as:
\begin{equation}\label{eq:nl}
\begin{aligned}
    \min_{W,Z} \sum_{k=1}^N \sum_{i=1}^d \sum_{t=1}^{T-1} 
    &\big\| X_{k,t+1}^{i} - h_{k,i}\big(f_i(P_{k,t};
    W_{0}^{i}),g_{k,i}(X_{k,t};W_{k}^{i})\big) 
    \big\|^2_2
    + (1 - \alpha) \sum_{i=1}^d \sum_{j=1}^{d+p} 
    \big\| (W_{0,j}^i, W_{1,j}^i,\ldots,W_{N,j}^i) 
    \big\|_2\\
    &+ \alpha\sum_{k=1}^N \sum_{i=1}^d \sum_{j=1}^{d}
    \big\| W_{k,j}^i \big\|_2 + \lambda_z \sum_{k=1}^N 
    \sum_{l=1}^{p} \sqrt{ \frac{1}{T}\sum_{t=1}^T 
    Z^2_{k,l,t}},
\end{aligned}
\end{equation}
where $W_0^{i}$ and $W_k^{i}$ denote the input-layer weights of $f_i$ and $g_{k,i}$. After optimization, the observed-variable block of $W_0^i$ is used to score the shared Granger causal structure. This readout recovers the Granger graph when the latent modules sufficiently absorb the confounding contribution. When absorption is incomplete, the shared readout may remain contaminated by residual confounding; under the stable-confounding assumption of Theorem~\ref{thm:two-sided}, however, such residual confounding is common across environments and cannot by itself generate environment-specific deviations, so edges with $\hat{\mathcal{I}}_{ij}\neq\emptyset$ remain certifiable through intervention-induced contrast. A non-zero environment-specific weight $W_{k,j}^i$ represents an inferred deviation of edge $j\to i$ in environment $k$, yielding the estimated intervention set $\hat{\mathcal{I}}_{ij} = \{k : W_{k,j}^{i} \neq 0\}$. Thus, GRACE estimates edge-level intervention patterns without requiring intervention labels or targets.

\begin{algorithm}[ht]
\caption{Proximal gradient updates for the structured penalty in Eq.~(\ref{eq:nl})}
\label{alg:prox}
\begin{algorithmic}[1]
\Require $\alpha>0$, step size $\eta>0$, gradient-updated block
$(V_{0,j}^{i},V_{1,j}^{i},\ldots,V_{N,j}^{i})$

\For{$k=1$ to $N$}
    \State $\widetilde V_{k,j}^{i}
    \gets \mathcal{S}_{\eta\alpha}(V_{k,j}^{i})$
\EndFor

\State $\widetilde V_{0,j}^{i} \gets V_{0,j}^{i}$

\State $(W_{0,j}^{i},\ldots,W_{N,j}^{i})
\gets
\mathcal{S}_{\eta(1-\alpha)}
(\widetilde V_{0,j}^{i},\ldots,\widetilde V_{N,j}^{i})$

\State \Return $(W_{0,j}^{i},\ldots,W_{N,j}^{i})$
\end{algorithmic}
\end{algorithm}

We optimize Eq.~(\ref{eq:nl}) using proximal gradient descent. Each iteration first takes a gradient step on the smooth prediction loss and then applies the structured proximal mapping to the input-layer weights. The inner shrinkage promotes sparse environment-specific deviations, while the outer group shrinkage promotes sparse edge blocks shared across environments. For any vector $v$, let
\begin{equation}
\mathcal{S}_{\tau}(v)
=
\left(1-\frac{\tau}{\|v\|_2}\right)_{+}v
\end{equation}
denote group soft-thresholding. Algorithm~\ref{alg:prox} gives the proximal steps for an observed input; the latent update and full derivation are deferred to Appendix~\ref{app:optimization}.

\subsection{Theoretical Analysis}
\label{sec:theory}

For an edge $j\to i$, define the population apparent lag-one coefficient matrix in environment $k$ as $\mathcal{A}_k:=\mathrm{Cov}_k(X_{t+1},X_t)\,\mathrm{Cov}_k(X_t)^{-1}$ and let $a_k=\mathcal{A}_k[i,j]$. Throughout this section, $W_0[i,j]$ and $\Delta_k[i,j]$ denote the true shared coefficient and intervention deviation associated with edge $j\to i$ in Eq.~(\ref{eq:env}), as opposed to the learned weights $W_0^i$ and $W_k^i$ of Section~\ref{mtd}. Under stable confounding,
\begin{equation}
 a_k = c_{ij}+\Delta_k[i,j],
 \qquad c_{ij}:=W_0[i,j]+\Gamma[i,j],
\label{eq:edge-apparent}
\end{equation}
where $\Gamma[i,j]$ is the contribution of the latent process to the apparent coefficient (Appendix~\ref{app:theory}), assumed invariant across environments (stable confounding). GRACE represents these coefficients by a shared component and sparse environment-specific deviations. Conditional on the edge block being active under the outer group penalty, the population sparse-deviation subproblem induced by Eq.~(\ref{eq:ll}) is
\begin{equation}
\min_{w,\{\delta_k\}}
\frac{1}{2}\sum_{k=1}^{N}(a_k-w-\delta_k)^2
+\lambda\sum_{k=1}^{N}|\delta_k|,
\label{eq:edge-decomp}
\end{equation}
where $\lambda>0$ denotes the effective scalar threshold induced by the environment-specific sparsity penalty in Eq.~(\ref{eq:ll}) after population-loss normalization. Let $\mathcal{I}_{ij}=\{k:\Delta_k[i,j]\neq0\}$ and $q=|\mathcal{I}_{ij}|$.

\begin{theorem}[Intervention-induced recovery and certification]
\label{thm:two-sided}
Assume stable confounding and structure-preserving interventions,
$\mathrm{supp}(\Delta_k)\subseteq\mathrm{supp}(W_0)$. If
$1\le q<N/2$ and
\begin{equation}
\min_{k\in\mathcal{I}_{ij}}|\Delta_k[i,j]|
>
\frac{N}{N-q}\lambda,
\label{eq:separation}
\end{equation}
then the minimizer of Eq.~(\ref{eq:edge-decomp}) exactly recovers the unknown intervention support,
\begin{equation}
\{k:\hat\delta_k\neq0\}=\mathcal{I}_{ij},
\end{equation}
and its shared component satisfies
\begin{equation}
|\hat w-c_{ij}|\le \frac{q}{N-q}\lambda<\lambda.
\end{equation}
Consequently, any recovered intervention deviation certifies $W_0[i,j]\neq0$ under the structure-preserving intervention assumption.
\end{theorem}

The theorem links the identification argument directly to the sparse decomposition used by GRACE. The lower bound $q\ge1$ supplies cross-environment contrast, while $q<N/2$ keeps the unperturbed regime in the majority and bounds the displacement of the shared estimate. The separation condition ensures that genuine intervention deviations survive soft-thresholding. If $q=0$, contrast-based certification provides no evidence against invariant confounding; if the unperturbed environments cease to be a majority, exact recovery is no longer guaranteed. The proof, failure cases, and the effect of the outer group penalty are given in Appendix~\ref{app:theory}.

\section{Experiments}

In this work, we compare our proposed method to state-of-the-art (SOTA) baselines, including GC~\citep{granger1969investigating}, CD-NOD~\citep{huang2020causal}, LPCMCI~\citep{gerhardus2020high}, NGC~\citep{tank2021neural}, eSRU~\citep{khanna2019economy}, CUTS~\citep{cheng2023cuts, cheng2024cuts+}, DyNOTEARS~\citep{pamfil2020dynotears}, GVAR~\citep{marcinkevivcs2021interpretable}, and JRNGC~\citep{zhou2024jacobian}. 

We evaluate graph recovery using the area under the receiver operating characteristic curve (AUROC) and the area under the precision--recall curve (AUPRC). AUROC summarizes ranking quality across positive and negative edges, whereas AUPRC is particularly informative for sparse graphs because it emphasizes precision and recall on the relatively rare positive edges.

\subsection{Experimental Setup}\label{exp-setup}
\noindent\textbf{Synthetic Time Series Data.}
We generate synthetic data from the model in Section~\ref{sec:setting} with $d=5$ observed variables and $p=1$ latent confounder. The Granger summary graph over the observed variables is drawn with edge probability $e=0.3$, and the latent confounder is a common cause of two randomly selected observed variables. Each observed variable at time $t+1$ depends on all $d+p$ variables at time $t$ plus standard Gaussian noise; the linear setting uses a first-order vector autoregression, and the nonlinear setting applies a Leaky ReLU with negative slope $0.01$. We generate three environments: two are observational, and one is subject to edge-level interventions that modify the corresponding coefficients. None of the methods is told which environment is intervened.

\noindent \textbf{Real-world Time Series Data.}
We also evaluate all competing causal discovery methods on the real-world benchmarks, including \textbf{1) Tennessee Eastman Process (TEP) dataset}~\citep{downs1993plant}, a widely used benchmark in time series anomaly detection with ground-truth causal graphs. The dataset contains 33 variables, 960 observations, and 21 predefined anomalies, each anomaly representing a distinct environment. We construct two versions based on TEP data: \textit{Conf-TEP (w/o Interventions)} uses only the normal time series, where a variable without parents but with multiple children is masked to induce latent confounding, while \textit{Conf-TEP} additionally treats anomaly regimes as heterogeneous environments with intervention-like mechanism shifts. We further evaluate our approach on \textbf{2) the Causal-Rivers dataset}~\citep{stein2025causalrivers}, a large-scale time series causal discovery benchmark constructed from river discharge measurements across Germany. We generate four representative subsets of \textit{Causal-Rivers}: \textit{Random}, containing diverse connected subgraphs and naturally heterogeneous regimes; \textit{Confounder}, which simulates latent confounding by removing a parent node to create an unobserved common cause; \textit{Flood}, corresponding to a strongly non-stationary period of extreme rainfall involving 42 nodes in the RiversElbeFlood region; and \textit{No Rain + Flood}, which concatenates stationary no-rain and non-stationary flood periods to create two contrasting environmental regimes. These real-world regime changes are used as proxies for intervention-like heterogeneity rather than as controlled interventions with known targets.

\subsection{Experimental Results}\label{exp}
\begin{table*}[ht]
\caption{Results on synthetic data and Conf-TEP. \best{Red}: best; \second{Blue}: second best.}
\label{tb:syn-tep}
\begin{center}
\resizebox{\linewidth}{!}{
\begin{small}
\begin{sc}
\begin{tabular}{lcccccccc}
\toprule
\multirow{3}{*}{Methods}
& \multicolumn{4}{c}{Synthetic}
& \multicolumn{4}{c}{Conf-TEP} \\
\cmidrule(lr){2-5} \cmidrule(lr){6-9}
& \multicolumn{2}{c}{Linear}
& \multicolumn{2}{c}{Nonlinear}
& \multicolumn{2}{c}{w/o Interventions}
& \multicolumn{2}{c}{with Interventions} \\
\cmidrule(lr){2-3} \cmidrule(lr){4-5} \cmidrule(lr){6-7} \cmidrule(lr){8-9}
& AUROC & AUPRC & AUROC & AUPRC & AUROC & AUPRC & AUROC & AUPRC \\
\midrule
GC        & 0.6667 & 0.3684 & 0.6111 & 0.3333 & 0.6524 & 0.0877 & 0.6215 & 0.0795 \\
CD-NOD    & 0.8611 & 0.4167 & 0.8056 & 0.2917 & 0.6046 & 0.1185 & 0.6096 & 0.0697 \\
LPCMCI    & 0.7540 & 0.5387 & 0.6032 & 0.4448 & 0.6214 & 0.0821 & 0.6735 & 0.0829 \\
NGC       & 0.8968 & 0.8187 & 0.6349 & 0.6063 & 0.6338 & \second{0.3140} & 0.6773 & \second{0.4100} \\
eSRU      & \second{0.9365} & \second{0.8333} & \second{0.9841} & \second{0.9617} & 0.6656 & 0.2246 & \second{0.7503} & 0.4074 \\
CUTS      & 0.9127 & 0.7997 & 0.8968 & 0.7742 & 0.6716 & 0.1864 & 0.6963 & 0.2884 \\
DyNOTEARS & 0.7817 & 0.5962 & 0.7579 & 0.5925 & \second{0.6883} & 0.2743 & 0.7022 & 0.3359 \\
GVAR      & 0.9206 & 0.8072 & 0.8016 & 0.7087 & 0.6733 & 0.2510 & 0.7112 & 0.2679 \\
JRNGC     & 0.7857 & 0.6445 & 0.8254 & 0.6828 & 0.6098 & 0.1877 & 0.6627 & 0.2620 \\
\textbf{\model{}} & \best{1.0000} & \best{1.0000} & \best{1.0000} & \best{1.0000} & \best{0.7855} & \best{0.3567} & \best{0.8121} & \best{0.4380} \\
\bottomrule
\end{tabular}
\end{sc}
\end{small}
}
\end{center}
\end{table*}
We evaluate the performance of GRACE across these datasets. The results in Table~\ref{tb:syn-tep} show that GRACE achieves the strongest performance in both linear and nonlinear synthetic settings. These results are consistent with the design of the synthetic benchmark, which isolates the shared-mechanism, latent-confounding, and intervention-deviation structure studied in this paper. While several deep learning-based baselines also perform competitively in synthetic settings, they are not designed to jointly account for latent confounding and unknown interventions, which limits their performance when both factors are present.

On the more challenging Conf-TEP settings, all methods experience performance degradation, reflecting the difficulty of Granger causal structure recovery under hidden confounding and sparse graphs. Although some baselines achieve relatively high AUROC, their lower AUPRC indicates difficulty recovering the comparatively rare positive edges. GRACE performs best both with and without intervention-like environments. The latter result does not contradict the lower-bound statement in Section~\ref{sec:theory}: that result concerns certification from cross-environment contrast, whereas GRACE may still recover edges when its latent module sufficiently accounts for confounding even in the absence of such contrast.

\begin{table*}[ht]
\caption{Results on Causal-Rivers. \best{Red}: best; \second{Blue}: second best.}
\label{tb:rivers}
\begin{center}
\resizebox{\linewidth}{!}{
\begin{small}
\begin{sc}
\begin{tabular}{lcccccccc}
\toprule
\multirow{2}{*}{Methods}
& \multicolumn{2}{c}{Random}
& \multicolumn{2}{c}{Confounder}
& \multicolumn{2}{c}{Flood}
& \multicolumn{2}{c}{No Rain + Flood} \\
\cmidrule(lr){2-3} \cmidrule(lr){4-5} \cmidrule(lr){6-7} \cmidrule(lr){8-9}
& AUROC & AUPRC & AUROC & AUPRC & AUROC & AUPRC & AUROC & AUPRC \\
\midrule
GC        & 0.5313 & 0.3750 & 0.5079 & 0.4415 & 0.7006 & 0.0847 & 0.7122 & 0.0940 \\
CD-NOD    & 0.5938 & 0.3432 & 0.5741 & 0.2778 & 0.5073 & 0.0368 & 0.4782 & 0.0241 \\
LPCMCI    & 0.5625 & 0.4381 & 0.5714 & 0.6095 & 0.7683 & 0.5631 & 0.6954 & 0.5456 \\
NGC       & 0.6042 & 0.5156 & 0.4762 & 0.5441 & \second{0.7747} & \best{0.5814} & 0.7627 & \second{0.5736} \\
eSRU      & 0.7778 & 0.7671 & \second{0.8571} & 0.8429 & 0.6925 & 0.3606 & 0.7627 & 0.3435 \\
CUTS      & 0.8194 & 0.8267 & 0.8254 & \second{0.8691} & 0.7247 & 0.3362 & 0.7513 & 0.4513 \\
DyNOTEARS & 0.8056 & 0.8313 & 0.7143 & 0.6301 & 0.7623 & 0.5599 & \second{0.7796} & 0.5701 \\
GVAR      & \second{0.9028} & 0.8817 & 0.6825 & 0.6372 & 0.6602 & 0.2949 & 0.6892 & 0.2983 \\
JRNGC     & 0.8958 & \best{0.9062} & 0.6508 & 0.6810 & 0.7414 & 0.3914 & 0.7062 & 0.2508 \\
\textbf{\model{}} & \best{0.9097} & \second{0.9038} & \best{0.9048} & \best{0.9341} & \best{0.7788} & \second{0.5750} & \best{0.7832} & \best{0.5839} \\
\bottomrule
\end{tabular}
\end{sc}
\end{small}
}
\end{center}
\end{table*}

The results in Table~\ref{tb:rivers} show that GRACE remains competitive across all four \textit{Causal-Rivers} subsets and attains the highest AUROC and AUPRC on the \textit{Confounder} subset. On \textit{Random}, its performance is comparable to the strongest baselines, while on \textit{Flood} and \textit{No Rain + Flood} it remains among the top-performing methods under pronounced non-stationarity and heterogeneous regimes. These results are consistent with the intended role of GRACE: modeling a shared mechanism while allowing sparse environment-specific deviations and latent components can be beneficial when both hidden confounding and distributional heterogeneity are present. Because the flood and anomaly regimes are not controlled interventions with known targets, we interpret these experiments as tests under intervention-like heterogeneity rather than direct validation of intervention recovery. 

\subsection{Ablation Study}\label{sec:ablation}
\begin{figure}[ht]
\centering
\includegraphics[width=0.7\linewidth]{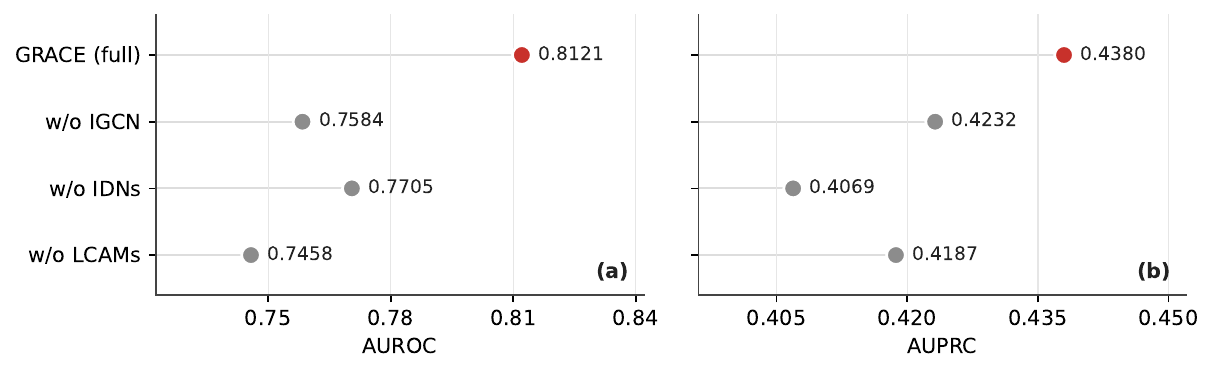}
\caption{Ablation of \model{} on Conf-TEP: (a) AUROC and (b) AUPRC of the full model and variants without the invariant Granger causal network (IGCN), intervention detection networks (IDNs), or latent confounder absorption modules (LCAMs).}
\label{fig:ablation}
\end{figure}

Figure~\ref{fig:ablation} evaluates the contribution of the main components of \model{} on Conf-TEP. Removing any individual component degrades both AUROC and AUPRC, indicating that the invariant mechanism, intervention-specific deviations, and latent confounding module each contribute to structure recovery. Removing the LCAMs leads to the largest drop in AUROC, consistent with the strong latent confounding present in Conf-TEP. Removing the IDNs produces the largest decrease in AUPRC, suggesting that explicitly modeling environment-specific interventions is particularly important for controlling false positives in the sparse-graph regime. Overall, the results support the complementary roles of latent-confounder modeling and intervention-aware structure learning.

\begin{figure*}[ht]
    \centering
    \includegraphics[width=\linewidth]{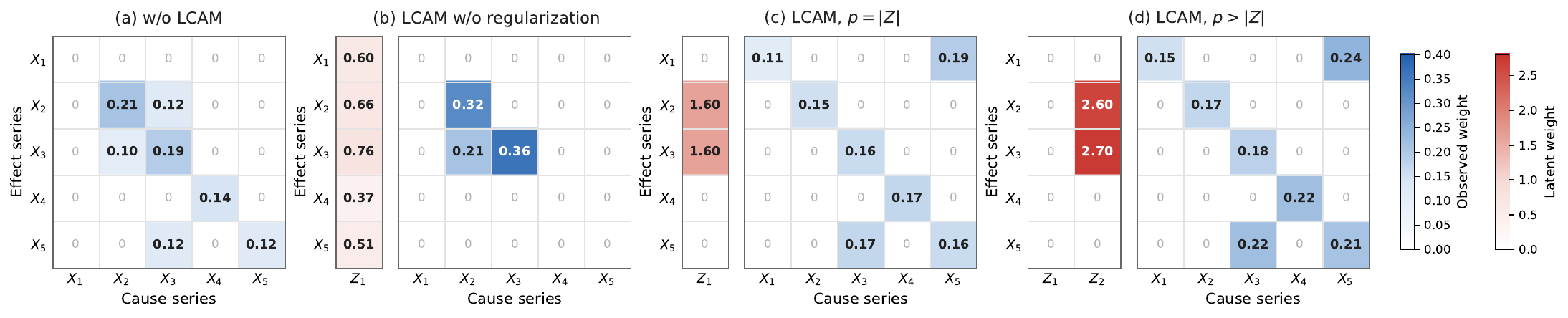}
    \caption{Ablation study of the latent confounder absorption modules (LCAMs). We compare four settings: (a) w/o LCAMs, (b) w/ LCAMs but without regularization, (c) w/ LCAMs with $p$ set to the true number of latent confounders ($p=|Z|$), and (d) w/ LCAMs with $p > |Z|$.}
    \label{fig:abl_LCAM}
\end{figure*}
Figure~\ref{fig:abl_LCAM} further examines the latent confounder module. Without LCAMs, the recovered graph contains substantial spurious structure. Introducing the latent module without regularization reduces these dependencies only partially, whereas adding the regularization term in Eq.~(\ref{eq:nl}) further suppresses spurious edges and improves recovery of the underlying Granger causal structure. Performance is strongest when the latent dimension $p$ is sufficiently large to capture the underlying confounding factors and $\lambda_z$ is appropriately regularized. When $p<|Z|$, the latent representation has insufficient capacity to absorb the confounding signal, leaving residual spurious dependencies.

\subsection{Sensitivity Analysis}\label{sec:sensitivity}
\begin{figure*}[ht]
\centering
\includegraphics[width=0.85\textwidth]{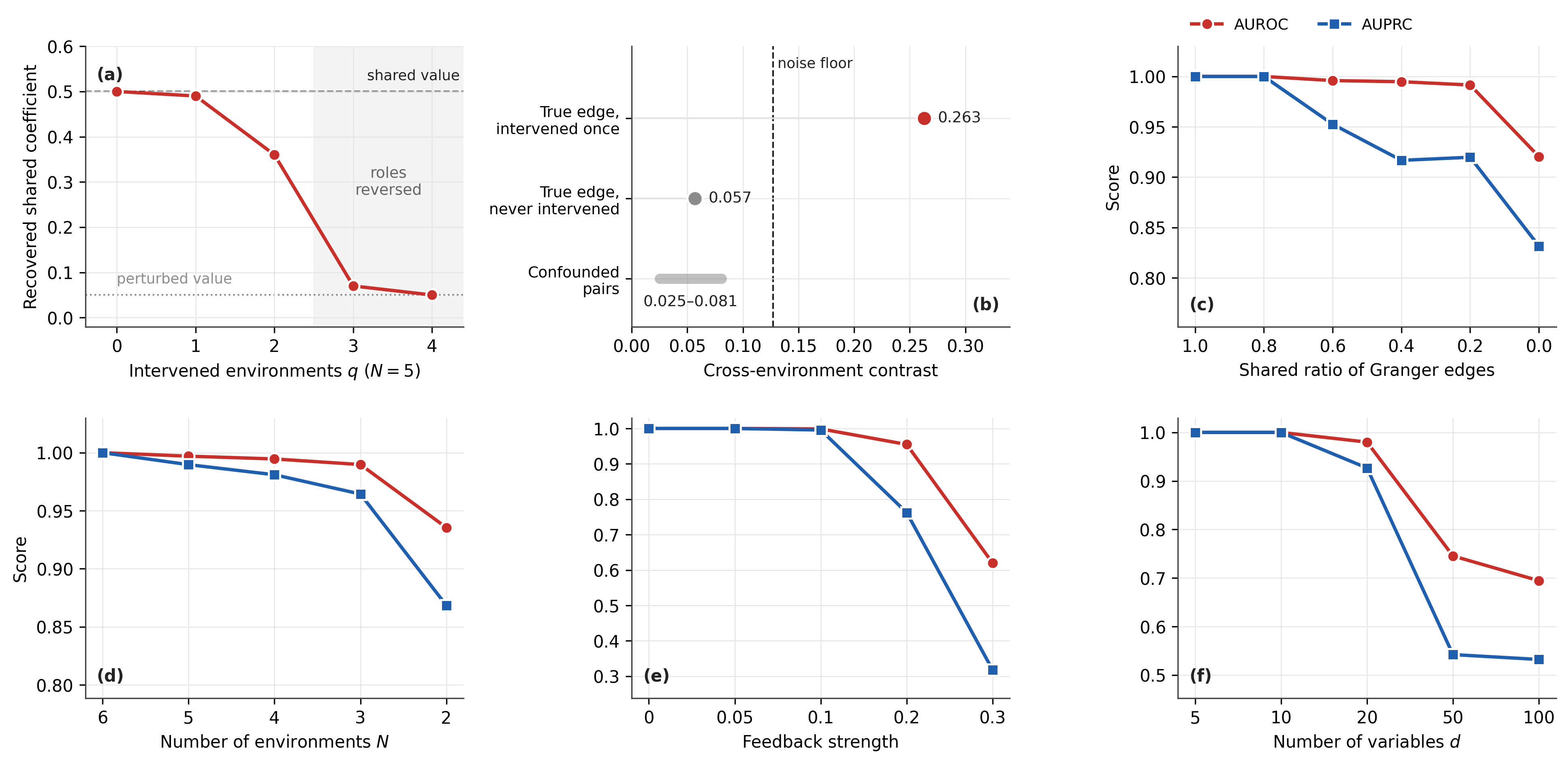}
\caption{Sensitivity of \model{} on synthetic data. (a) Recovered shared coefficient as the number of intervened environments varies. (b) Cross-environment contrast for different edge types under strong confounding. (c)-(f) AUROC and AUPRC under varying shared structure, number of environments, latent feedback strength, and graph size.}
\label{fig:sensitivity}
\end{figure*}

We vary one factor of the synthetic data-generating process at a time while holding the remaining settings fixed (Figure~\ref{fig:sensitivity}). We first examine whether the empirical behavior follows the theoretical identification boundary. When $N=5$ and a single edge is intervened in $q$ environments, replacing its shared coefficient $0.5$ with $0.05$, the recovered shared coefficient remains close to the true value for $q\le2$ ($0.50$, $0.49$, and $0.36$ for $q=0,1,2$), but collapses toward the perturbed regime once $q\ge3$ ($0.07$ and $0.05$), illustrating the loss of the majority reference predicted by Theorem~\ref{thm:two-sided}. A complementary view is provided by the cross-environment contrast under strong confounding: the intervened edge exhibits substantially larger contrast ($0.263$) than the noise floor ($0.127$), whereas the never-intervened edge ($0.057$) and confounded non-edges ($0.025$--$0.081$) remain below it. This supports the central role of intervention-induced variation in certifying affected edges when invariant confounding persists. 

We next vary the shared structure, the number of environments, the feedback into the latent process, and the graph size. When the fraction of Granger edges shared across environments decreases, performance remains near ceiling at high shared fractions and AUROC stays above $0.99$ down to $0.2$, indicating that partially shared structure remains exploitable. Reducing the number of environments causes a gradual decline, with the largest drop at $N=2$ where AUPRC falls to $0.87$; in this regime, perturbed and unperturbed environments cannot form the strict-majority configuration required by the theory. Similarly, mild feedback from observed variables into the latent process has little effect, whereas stronger feedback beyond $0.1$ leads to substantial degradation, reaching an AUPRC of $0.32$ at strength $0.3$, consistent with increasing entanglement between latent and observed dynamics. Finally, we examine scalability by increasing the graph size while keeping hyperparameters fixed: performance remains strong up to $d=20$ but gradually degrades for larger graphs, with AUPRC falling to approximately $0.53$ at $d=100$, reflecting the limitations of the current $d(1+N)$ predictor parameterization.
\subsection{Complexity Analysis}
\label{sec:complexity}

Each invariant function $f_i(\cdot)$, $i=1,\dots,d$, and each environment-specific function $g_{k,i}(\cdot)$, $i=1,\dots,d$, $k=1,\dots,N$, is implemented as an independent MLP. Ignoring lower-order bias and output-layer terms, the network parameters scale as $O(Hd(d+p)+HNd^2)$, which reduces to $O(HNd^2)$ when $p\ll d$ and $N$ is not negligible. Because the LCAM variables $Z_{k,t}\in\mathbb{R}^p$ are directly optimized, they contribute an additional $O(NpT)$ trainable quantities. The per-step network cost is $O(HNd^2)$, and processing a sequence of length $T$ therefore costs $O(THNd^2)$, in addition to storing and updating the $O(NpT)$ latent representations.

\subsection{Hyperparameter Robustness}

The regularization strengths control the trade-off between suppressing spurious structure and preserving genuine Granger relations and intervention-specific deviations. Across the reported experiments, we use $\lambda_z\in[0.1,0.5]$, with values selected from $\{0.1,0.2,0.3,0.35,0.5\}$, and $\alpha\in[10^{-3},10^{-2}]$. Performance remains stable over these ranges in our experiments, suggesting that the method is not highly sensitive to small changes in either regularization strength.

\section{Limitations and Future Work}
Our theoretical guarantee is currently limited to a linear population setting with stable confounding, leaving nonlinear identifiability and finite-sample analysis as important directions for future work. Moreover, LCAMs are intended to account for latent confounding in support of Granger structure recovery, rather than to recover the latent process itself. Finally, the current architecture scales quadratically with the number of observed variables, motivating more parameter-efficient designs for high-dimensional time series.

\section{Conclusion}
This work introduces GRACE for Granger causal structure learning in the joint presence of latent confounding and unknown interventions. We show that cross-environment invariance alone can be ambiguous when latent confounding induces stable predictive dependencies, and use intervention-induced heterogeneity as complementary identifying contrast. GRACE operationalizes this idea by jointly modeling a shared Granger mechanism, sparse environment-specific deviations, and regularized latent representations. Under a linear population model, we establish a sufficient per-edge condition under which the sparse decomposition underlying GRACE exactly recovers unknown intervention support; the recovered deviations then provide a certificate for the corresponding Granger edge. Experiments on synthetic and real-world time series support the practical value of this decomposition under latent confounding and heterogeneous environments.

\bibliography{main}

@article{brouillard2020differentiable,
  title={Differentiable causal discovery from interventional data},
  author={Brouillard, Philippe and Lachapelle, S{\'e}bastien and Lacoste, Alexandre and Lacoste-Julien, Simon and Drouin, Alexandre},
  journal={Advances in Neural Information Processing Systems},
  volume={33},
  pages={21865--21877},
  year={2020}
}

@inproceedings{mameche2024identifying,
  title={Identifying confounding from causal mechanism shifts},
  author={Mameche, Sarah and Vreeken, Jilles and Kaltenpoth, David},
  booktitle={International Conference on Artificial Intelligence and Statistics},
  pages={4897--4905},
  year={2024},
  organization={PMLR}
}

@article{scholkopf2021toward,
  title={Toward causal representation learning},
  author={Sch{\"o}lkopf, Bernhard and Locatello, Francesco and Bauer, Stefan and Ke, Nan Rosemary and Kalchbrenner, Nal and Goyal, Anirudh and Bengio, Yoshua},
  journal={Proceedings of the IEEE},
  volume={109},
  number={5},
  pages={612--634},
  year={2021},
  publisher={IEEE}
}

@article{granger1969investigating,
  title={Investigating causal relations by econometric models and cross-spectral methods},
  author={Granger, Clive WJ},
  journal={Econometrica: Journal of the Econometric Society},
  pages={424--438},
  year={1969},
  publisher={JSTOR}
}

@article{tank2021neural,
  title={Neural granger causality},
  author={Tank, Alex and Covert, Ian and Foti, Nicholas and Shojaie, Ali and Fox, Emily B},
  journal={IEEE Transactions on Pattern Analysis and Machine Intelligence},
  volume={44},
  number={8},
  pages={4267--4279},
  year={2021},
  publisher={IEEE}
}

@article{marcinkevivcs2021interpretable,
  title={Interpretable models for granger causality using self-explaining neural networks},
  author={Marcinkevi{\v{c}}s, Ri{\v{c}}ards and Vogt, Julia E},
  journal={International Conference on Learning Representations},
  year={2021}
}

@article{khanna2019economy,
  title={Economy statistical recurrent units for inferring nonlinear granger causality},
  author={Khanna, Saurabh and Tan, Vincent YF},
  journal={International Conference on Learning Representations},
  year={2019}
}

@article{cheng2023cuts,
  title={CUTS: Neural causal discovery from irregular time-series data},
  author={Cheng, Yuxiao and Yang, Runzhao and Xiao, Tingxiong and Li, Zongren and Suo, Jinli and He, Kunlun and Dai, Qionghai},
  journal={International Conference on Learning Representations},
  year={2023}
}

@article{zhou2024jacobian,
  title={Jacobian regularizer-based neural Granger causality},
  author={Zhou, Wanqi and Bai, Shuanghao and Yu, Shujian and Zhao, Qibin and Chen, Badong},
  journal={International Conference on Machine Learning},
  year={2024}
}

@inproceedings{cheng2024cuts+,
  title={CUTS+: High-dimensional causal discovery from irregular time-series},
  author={Cheng, Yuxiao and Li, Lianglong and Xiao, Tingxiong and Li, Zongren and Suo, Jinli and He, Kunlun and Dai, Qionghai},
  booktitle={Proceedings of the AAAI Conference on Artificial Intelligence},
  volume={38},
  pages={11525--11533},
  year={2024}
}

@inproceedings{pamfil2020dynotears,
  title={Dynotears: Structure learning from time-series data},
  author={Pamfil, Roxana and Sriwattanaworachai, Nisara and Desai, Shaan and Pilgerstorfer, Philip and Georgatzis, Konstantinos and Beaumont, Paul and Aragam, Bryon},
  booktitle={International Conference on Artificial Intelligence and Statistics},
  pages={1595--1605},
  year={2020},
  organization={PMLR}
}

@article{squires2023causal,
  title={Causal structure learning: A combinatorial perspective},
  author={Squires, Chandler and Uhler, Caroline},
  journal={Foundations of Computational Mathematics},
  volume={23},
  number={5},
  pages={1781--1815},
  year={2023},
  publisher={Springer}
}

@inproceedings{malinsky2018causal,
  title={Causal structure learning from multivariate time series in settings with unmeasured confounding},
  author={Malinsky, Daniel and Spirtes, Peter},
  booktitle={Proceedings of 2018 ACM SIGKDD Workshop on Causal Discovery},
  pages={23--47},
  year={2018},
  organization={PMLR}
}

@article{lecun2015deep,
  title={Deep learning},
  author={LeCun, Yann and Bengio, Yoshua and Hinton, Geoffrey},
  journal={Nature},
  volume={521},
  number={7553},
  pages={436--444},
  year={2015},
  publisher={Nature Publishing Group UK London}
}

@article{huang2020causal,
  title={Causal discovery from heterogeneous/nonstationary data},
  author={Huang, Biwei and Zhang, Kun and Zhang, Jiji and Ramsey, Joseph and Sanchez-Romero, Ruben and Glymour, Clark and Sch{\"o}lkopf, Bernhard},
  journal={Journal of Machine Learning Research},
  volume={21},
  number={89},
  pages={1--53},
  year={2020}
}

@article{arjovsky2019invariant,
  title={Invariant risk minimization},
  author={Arjovsky, Martin and Bottou, L{\'e}on and Gulrajani, Ishaan and Lopez-Paz, David},
  journal={arXiv preprint arXiv:1907.02893},
  year={2019}
}

@inproceedings{liu2023causal,
  title={Causal structure learning for latent intervened non-stationary data},
  author={Liu, Chenxi and Kuang, Kun},
  booktitle={International Conference on Machine Learning},
  pages={21756--21777},
  year={2023},
  organization={PMLR}
}

@inproceedings{gao2022idyno,
  title={IDYNO: Learning nonparametric DAGs from interventional dynamic data},
  author={Gao, Tian and Bhattacharjya, Debarun and Nelson, Elliot and Liu, Miao and Yu, Yue},
  booktitle={International Conference on Machine Learning},
  pages={6988--7001},
  year={2022},
  organization={PMLR}
}

@inproceedings{arnold2007temporal,
  title={Temporal causal modeling with graphical granger methods},
  author={Arnold, Andrew and Liu, Yan and Abe, Naoki},
  booktitle={Proceedings of the 13th ACM SIGKDD International Conference on Knowledge Discovery and Data Mining},
  pages={66--75},
  year={2007}
}

@inproceedings{han2025root,
  title={Root cause analysis of anomalies in multivariate time series through Granger causal discovery},
  author={Han, Xiao and Absar, Saima and Zhang, Lu and Yuan, Shuhan},
  booktitle={International Conference on Learning Representations},
  year={2025}
}

@inproceedings{chandrasekaran2010latent,
  title={Latent variable graphical model selection via convex optimization},
  author={Chandrasekaran, Venkat and Parrilo, Pablo A and Willsky, Alan S},
  booktitle={2010 48th Annual Allerton Conference on Communication, Control, and Computing (Allerton)},
  pages={1610--1613},
  year={2010},
  organization={IEEE}
}

@article{colombo2012learning,
  title={Learning high-dimensional directed acyclic graphs with latent and selection variables},
  author={Colombo, Diego and Maathuis, Marloes H and Kalisch, Markus and Richardson, Thomas S},
  journal={The Annals of Statistics},
  pages={294--321},
  year={2012},
  publisher={JSTOR}
}

@inproceedings{ogarrio2016hybrid,
  title={A hybrid causal search algorithm for latent variable models},
  author={Ogarrio, Juan Miguel and Spirtes, Peter and Ramsey, Joe},
  booktitle={Conference on Probabilistic Graphical Models},
  pages={368--379},
  year={2016},
  organization={PMLR}
}

@article{tibshirani1996regression,
  title={Regression shrinkage and selection via the lasso},
  author={Tibshirani, Robert},
  journal={Journal of the Royal Statistical Society Series B: Statistical Methodology},
  volume={58},
  number={1},
  pages={267--288},
  year={1996},
  publisher={Oxford University Press}
}

@article{yuan2006model,
  title={Model selection and estimation in regression with grouped variables},
  author={Yuan, Ming and Lin, Yi},
  journal={Journal of the Royal Statistical Society Series B: Statistical Methodology},
  volume={68},
  number={1},
  pages={49--67},
  year={2006},
  publisher={Oxford University Press}
}

@book{spirtes2000causation,
  title={Causation, prediction, and search},
  author={Spirtes, Peter and Glymour, Clark N and Scheines, Richard},
  year={2000},
  publisher={MIT press}
}

@article{zhang2008completeness,
  title={On the completeness of orientation rules for causal discovery in the presence of latent confounders and selection bias},
  author={Zhang, Jiji},
  journal={Artificial Intelligence},
  volume={172},
  number={16-17},
  pages={1873--1896},
  year={2008},
  publisher={Elsevier}
}

@article{gerhardus2020high,
  title={High-recall causal discovery for autocorrelated time series with latent confounders},
  author={Gerhardus, Andreas and Runge, Jakob},
  journal={Advances in Neural Information Processing Systems},
  volume={33},
  pages={12615--12625},
  year={2020}
}

@article{entner2010causal,
  title={On causal discovery from time series data using FCI},
  author={Entner, Doris and Hoyer, Patrik O},
  journal={Probabilistic Graphical Models},
  volume={16},
  year={2010}
}

@inproceedings{yang2018characterizing,
  title={Characterizing and learning equivalence classes of causal dags under interventions},
  author={Yang, Karren and Katcoff, Abigail and Uhler, Caroline},
  booktitle={International Conference on Machine Learning},
  pages={5541--5550},
  year={2018},
  organization={PMLR}
}

@article{stein2025causalrivers,
  title={CausalRivers - Scaling up benchmarking of causal discovery for real-world time-series},
  author={Stein, Gideon and Shadaydeh, Maha and Blunk, Jan and Penzel, Niklas and Denzler, Joachim},
  journal={International Conference on Learning Representations},
  year={2025}
}

@article{peters2016causal,
  title={Causal inference by using invariant prediction: identification and confidence intervals},
  author={Peters, Jonas and B{\"u}hlmann, Peter and Meinshausen, Nicolai},
  journal={Journal of the Royal Statistical Society Series B: Statistical Methodology},
  volume={78},
  number={5},
  pages={947--1012},
  year={2016},
  publisher={Oxford University Press}
}

@article{downs1993plant,
  title={A plant-wide industrial process control problem},
  author={Downs, James J and Vogel, Ernest F},
  journal={Computers \& Chemical Engineering},
  volume={17},
  number={3},
  pages={245--255},
  year={1993},
  publisher={Elsevier}
}

@inproceedings{zhang2024learning,
  title={Learning flexible time-windowed granger causality integrating heterogeneous interventional time series data},
  author={Zhang, Ziyi and Ren, Shaogang and Qian, Xiaoning and Duffield, Nick},
  booktitle={Proceedings of the 30th ACM SIGKDD Conference on Knowledge Discovery and Data Mining},
  pages={4408--4418},
  year={2024}
}

@article{heinze2018invariant,
  title={Invariant causal prediction for nonlinear models},
  author={Heinze-Deml, Christina and Peters, Jonas and Meinshausen, Nicolai},
  journal={Journal of Causal Inference},
  volume={6},
  number={2},
  pages={20170016},
  year={2018},
  publisher={De Gruyter}
}

@article{pfister2019invariant,
  title={Invariant causal prediction for sequential data},
  author={Pfister, Niklas and B{\"u}hlmann, Peter and Peters, Jonas},
  journal={Journal of the American Statistical Association},
  volume={114},
  number={527},
  pages={1264--1276},
  year={2019},
  publisher={Taylor \& Francis}
}
\bibliographystyle{tmlr}

\appendix
\section{Proofs and Additional Identifiability Analysis}
\label{app:theory}

\subsection{Population apparent coefficients}
Let
\begin{equation}
\mathcal{A}_k
:=
\mathrm{Cov}_k(X_{t+1},X_t)\,\mathrm{Cov}_k(X_t)^{-1}.
\end{equation}
Under the linear latent-variable model,
\begin{equation}
\mathcal{A}_k
=
W_0+\Delta_k+\Gamma_k,
\qquad
\Gamma_k
:=
W_{X_{t+1}Z_t}^{\top}\,\mathrm{Cov}_k(Z_t,X_t)\,\mathrm{Cov}_k(X_t)^{-1}.
\label{eq:app-apparent}
\end{equation}
Thus each environment identifies only the combined apparent coefficient. The analysis assumes that the confounding contribution is stable, $\Gamma_k=\Gamma$, and that interventions are structure preserving, $\mathrm{supp}(\Delta_k)\subseteq\mathrm{supp}(W_0)$.

\subsection{Sparse-deviation decomposition}
For a fixed edge, write $a_k=c+\Delta_k$ with $c=W_0[i,j]+\Gamma[i,j]$. Conditional on the corresponding edge block remaining active under the outer group penalty, the environment-specific part of Eq.~(\ref{eq:ll}) reduces at population level to
\begin{equation}
F(w,\delta)
=
\frac12\sum_{k=1}^{N}(a_k-w-\delta_k)^2
+\lambda\sum_{k=1}^{N}|\delta_k|.
\label{eq:app-edge-decomp}
\end{equation}
For fixed $w$, the minimizer is
\begin{equation}
\hat\delta_k(w)
=
\operatorname{soft}_{\lambda}(a_k-w),
\end{equation}
where $\operatorname{soft}_{\lambda}(x)=\mathrm{sign}(x)(|x|-\lambda)_+$. Equivalently, the residual
\begin{equation}
r_k(w)
:=a_k-w-\hat\delta_k(w)
=
\operatorname{clip}(a_k-w,-\lambda,\lambda).
\label{eq:residual-clip}
\end{equation}
The first-order condition for the shared component is
\begin{equation}
\sum_{k=1}^{N}r_k(\hat w)=0.
\label{eq:w-kkt}
\end{equation}

\begin{lemma}[Majority-regime bound]
\label{lem:majority-bound}
Let $q=|\mathcal{I}_{ij}|$ and $m=N-q$. If $q<m$, then any minimizer of Eq.~(\ref{eq:app-edge-decomp}) satisfies
\begin{equation}
|\hat w-c|\le \frac{q}{m}\lambda<\lambda.
\label{eq:w-bound}
\end{equation}
\end{lemma}

\begin{proof}
Suppose first that $\hat w-c\ge\lambda$. For each of the $m$ unperturbed environments, Eq.~(\ref{eq:residual-clip}) gives $r_k(\hat w)=-\lambda$. The remaining $q$ residuals are each at most $\lambda$, so
$\sum_k r_k(\hat w)\le-(m-q)\lambda<0$, contradicting Eq.~(\ref{eq:w-kkt}). The case $c-\hat w\ge\lambda$ is symmetric; hence $|\hat w-c|<\lambda$. Therefore every unperturbed residual equals $c-\hat w$, while each intervened residual has magnitude at most $\lambda$. Using Eq.~(\ref{eq:w-kkt}),
\begin{equation}
m|c-\hat w|
=
\left|\sum_{k\in\mathcal{I}_{ij}}r_k(\hat w)\right|
\le q\lambda,
\end{equation}
which yields Eq.~(\ref{eq:w-bound}).
\end{proof}

\begin{lemma}[Exact intervention-support recovery]
\label{lem:support-recovery}
Under the conditions of Lemma~\ref{lem:majority-bound}, if
\begin{equation}
\min_{k\in\mathcal{I}_{ij}}|\Delta_k[i,j]|
>
\lambda+\frac{q}{N-q}\lambda
=
\frac{N}{N-q}\lambda,
\end{equation}
then
\begin{equation}
\{k:\hat\delta_k\neq0\}=\mathcal{I}_{ij}.
\end{equation}
\end{lemma}

\begin{proof}
For $k\notin\mathcal{I}_{ij}$, $a_k=c$, and Lemma~\ref{lem:majority-bound} gives $|a_k-\hat w|=|c-\hat w|<\lambda$, hence $\hat\delta_k=0$. For $k\in\mathcal{I}_{ij}$,
\begin{equation}
|a_k-\hat w|
\ge
|\Delta_k[i,j]|-|c-\hat w|
>
\lambda,
\end{equation}
so soft-thresholding yields $\hat\delta_k\neq0$.
\end{proof}

\subsection{Proof of Theorem~\ref{thm:two-sided}}
\begin{proof}
Because $q<N/2$, we have $q<N-q$, so Lemma~\ref{lem:majority-bound} applies and gives the stated bound on $|\hat w-c_{ij}|$. The separation condition in Eq.~(\ref{eq:separation}) then invokes Lemma~\ref{lem:support-recovery}, yielding exact recovery of $\mathcal{I}_{ij}$. Since $q\ge1$, exact recovery implies that at least one nonzero intervention deviation exists. Under $\mathrm{supp}(\Delta_k)\subseteq\mathrm{supp}(W_0)$, any such deviation implies $W_0[i,j]\neq0$, certifying the Granger edge.
\end{proof}

\subsection{Boundary regimes and scope}
If $\mathcal{I}_{ij}=\emptyset$, all apparent coefficients equal $c_{ij}=W_0[i,j]+\Gamma[i,j]$, so cross-environment variation alone cannot distinguish a direct Granger dependency from an invariant confounded dependency. This is a limitation of contrast-based certification, not a general impossibility result for latent-variable Granger causal discovery.

When $q\ge N/2$, the unperturbed observations no longer form a strict majority, and Lemma~\ref{lem:majority-bound} no longer guarantees that the shared estimate remains within one threshold of the unperturbed baseline. Exact intervention-support recovery can therefore fail. The theorem provides a sufficient recovery regime rather than a claim that every configuration outside it must fail.

The analysis above isolates the sparse deviation term that connects GRACE to unknown intervention recovery. The full objective in Eq.~(\ref{eq:ll}) also contains the cross-environment group penalty used for edge-block selection. The theorem is conditional on the relevant block being active; characterizing joint block selection and deviation recovery under the full nested penalty requires additional assumptions and is left for future work. Likewise, the nonlinear objective in Eq.~(\ref{eq:nl}) is a practical extension rather than one covered by the linear population guarantee.

\section{Proximal Gradient Descent Updates}
\label{app:optimization}

We write the nonlinear objective as
\begin{equation}
\min_{W,Z}\;\mathcal{L}(W,Z)+\Omega(W,Z),
\end{equation}
where $\mathcal{L}$ is the prediction loss in Eq.~(\ref{eq:nl}) and
\begin{equation}
\begin{aligned}
\Omega(W,Z)
&= \lambda_z \sum_{k=1}^N \sum_{l=1}^{p}
\sqrt{ \tfrac{1}{T}\sum_{t=1}^T Z^2_{k,l,t}}
+ (1 - \alpha) \sum_{i=1}^d \sum_{j=1}^{d+p}
\big\| W_{j}^{i} \big\|_2\\
&\quad + \alpha\sum_{k=1}^N \sum_{i=1}^d \sum_{j=1}^{d}
\big\| W_{k,j}^i \big\|_2,
\end{aligned}
\end{equation}
with $W_{j}^{i}:=(W_{0,j}^i,W_{1,j}^i,\dots,W_{N,j}^i)$. For any $v\in\mathbb{R}^m$ and $\tau>0$, define the group soft-thresholding operator
\begin{equation}
\mathcal{S}_{\tau}(v)
=
\begin{cases}
\left(1-\dfrac{\tau}{\|v\|_2}\right)v, & \|v\|_2>\tau,\\[4pt]
0, & \|v\|_2\le\tau.
\end{cases}
\end{equation}

\paragraph{Latent update.}
For each $(k,l)$, after a gradient step
\begin{equation}
\widetilde{Z}_{k,l,:}^{(r)}
=
Z_{k,l,:}^{(r)}
-\eta\nabla_{Z_{k,l,:}}\mathcal{L}(W^{(r)},Z^{(r)}),
\end{equation}
the proximal update is
\begin{equation}
Z_{k,l,:}^{(r+1)}
=
\mathcal{S}_{\eta\lambda_z/\sqrt{T}}
\!\left(\widetilde{Z}_{k,l,:}^{(r)}\right).
\end{equation}

\paragraph{Weight update.}
Let
\begin{equation}
u_{k,j}^{i,(r)}
=
W_{k,j}^{i,(r)}
-\eta\nabla_{W_{k,j}^i}\mathcal{L}(W^{(r)},Z^{(r)}),
\qquad k=0,\dots,N.
\end{equation}
For an observed input $j\le d$, first shrink the environment-specific components,
\begin{equation}
\widetilde{u}_{0,j}^{i,(r)}=u_{0,j}^{i,(r)},
\qquad
\widetilde{u}_{k,j}^{i,(r)}
=
\mathcal{S}_{\eta\alpha}\!\left(u_{k,j}^{i,(r)}\right),\quad k\ge1,
\end{equation}
and then jointly shrink the full block,
\begin{equation}
W_{j}^{i,(r+1)}
=
\mathcal{S}_{\eta(1-\alpha)}
\!\left(\widetilde{u}_{j}^{i,(r)}\right).
\end{equation}
For a latent input $j>d$, the block contains only $W_{0,j}^i$ and is updated by the outer shrinkage alone.

\end{document}